\documentclass[12pt]{article}
\usepackage[margin=1in]{geometry}
\usepackage{microtype}
\usepackage{graphicx}
\usepackage{booktabs}
\usepackage{amsmath,amssymb,amsthm,mathtools,bm}
\usepackage{enumitem}
\usepackage[dvipsnames,table]{xcolor}
\usepackage{tikz}
\usetikzlibrary{calc,positioning,backgrounds,shapes.geometric,arrows.meta}
\usepackage{hyperref}
\usepackage{url}
\usepackage{natbib}

\newtheorem{theorem}{Theorem}
\newtheorem{lemma}{Lemma}

\newtheorem{corollary}{Corollary}
\theoremstyle{definition}
\newtheorem{definition}{Definition}

\theoremstyle{remark}
\newtheorem{remark}{Remark}

\newcommand{\E}{\mathbb{E}}

\newcommand{\TV}{\mathrm{D_{TV}}}

\title{Regret-Aware Policy Optimization: Environment-Level Memory\\
for Replay Suppression under Delayed Harm}
\author{Prakul Sunil Hiremath \\
\small Department of Computer Science and Engineering, VTU, Belagavi, India \\
\small Aliens on Earth (AoE) Autonomous Research Group, Belagavi, India \\
\small \href{mailto:prakulhiremath@vtu.ac.in}{\texttt{prakulhiremath@vtu.ac.in}}
}
\date{}

\begin{document}
\maketitle

\begin{abstract}
Safety in reinforcement learning (RL) is often enforced by objective shaping (e.g., Lagrangian penalties) while keeping the environment response to observable state--action pairs stationary. With delayed harms, this can create \textbf{replay}: after a washout period, reintroducing the same stimulus under a matched observable configuration reproduces a similar harmful cascade because the observable transition law is unchanged. We introduce the \textbf{Replay Suppression Diagnostic (RSD)}, a controlled exposure--washout--replay protocol that fixes stimulus identity and resets observable state and agent memory, while freezing the policy; only environment-side memory is allowed to persist. We prove a no-go result: under stationary observable kernels, replay-phase re-amplification cannot be structurally reduced without a persistent shift in replay-time action distributions. Motivated by platform-mediated systems, we propose \textbf{Regret-Aware Policy Optimization (RAPO)}, which adds persistent harm-trace and scar fields and uses them to apply a bounded, mass-preserving transition reweighting that reduces reachability of historically harmful regions. On graph diffusion tasks (50--1000 nodes), RAPO consistently suppresses replay; on 250-node graphs it reduces RAG from 0.98 (PM-ST control) to 0.33 while retaining 82\% task return. A counterfactual that disables deformation only during replay restores re-amplification (RAG: 0.91), isolating transition deformation as the causal mechanism.
\end{abstract}

\section{Introduction}

Content recommendation systems face a challenging safety failure mode: harmful content (misinformation, extremism) generates short-term engagement but causes delayed negative outcomes—user churn, reputation damage, regulatory penalties. Standard safe RL applies transient penalties when harm signals arrive, temporarily suppressing similar content. However, once penalties decay and similar contexts recur, systems with stationary observable transitions can reproduce the same harmful cascade under the same policy. We call this \textbf{replay}.

Replay arises from a structural mismatch: transient penalties modify objectives while environment responses to observable inputs remain history-invariant. This differs from exploration failures or distribution shift—it occurs under \emph{controlled conditions} (same stimulus, same observable state) after penalty decay.

\paragraph{Why existing approaches are insufficient.}
Constrained RL (CPO~\citep{achiam2017constrained}, Lagrangian methods) shapes objectives but leaves observable transitions stationary. Policy-side memory (recurrent policies, history features) changes action selection but doesn't alter environment responses to given $(x,a)$ pairs. Shields~\citep{alshiekh2018safe} prevent replay via persistent global avoidance, which can be overly conservative.

\paragraph{Our approach: Platform-mediated transition deformation.}
We target \textbf{platform-mediated systems}—recommendation engines, network routers, warehouse controllers, digital twins—where a centralized component controls routing or exposure. In these settings, we can deform transition kernels based on accumulated harm history while agents act on the same observable inputs. This is not ``changing physics''; it's implementing safety layers that gate transitions into regions with delayed-harm histories.

\paragraph{Contributions.}
We introduce the \textbf{Replay Suppression Diagnostic (RSD)}, an exposure-decay-replay protocol that isolates re-amplification under observable-matched conditions and frozen-policy evaluation (Section~\ref{sec:rsd}). We prove stationary observable kernels cannot structurally suppress replay without changing action distributions (Theorem\ref{thm:no_go_replay}), motivating environment-level intervention. We propose \textbf{RAPO}, which uses persistent harm-trace ($G$, decaying) and scar ($H$, irreversible) fields to deform transitions (Section~\ref{sec:method}), with theoretical guarantees on odds contraction (Lemma~\ref{lem:odds_contraction}) and utility preservation (Lemma~\ref{lem:utility_bound}). On graph diffusion (50-1000 nodes), RAPO reduces replay 67--83\% (RAG: 0.33 vs 1.02 for stationary baselines) while retaining 78--91\% utility. A policy-memory control (PM-ST) with matched observations shows no suppression (RAG: 0.98), and disabling deformation only during replay restores re-amplification (RAG: 0.91), providing causal evidence (Section~\ref{sec:experiments}).

\section{Problem Setup: Replay under Delayed Harm}
\label{sec:problem}

\subsection{Observable Dynamics with Delayed Harm}

The nominal environment follows stationary observable dynamics $x_{t+1} \sim P_0(\cdot \mid x_t, a_t)$ with reward $r(x_t,a_t)$, where $x_t \in \mathcal{X}$ is the observable state and $a_t \in \mathcal{A}$ is the action.
We allow the environment to maintain latent memory $\xi_t$ (e.g., logs, throttling state, or safety filters) that is not included in $x_t$.

\paragraph{Stationary-observable baseline vs.\ environment memory.}
In the \emph{stationary-observable baseline}, $\xi_t$ may evolve but does not affect the conditional law of the next observable state:
for any fixed stimulus identity $z$ and all $\xi_t$,
\begin{equation}
P(x_{t+1}\in \cdot \mid x_t,\xi_t,a_t;z)=P_0(\cdot\mid x_t,a_t).
\label{eq:stationary_observable_problemsetup}
\end{equation}
RAPO \emph{violates}~\eqref{eq:stationary_observable_problemsetup} by making the observable kernel depend on persistent environment-side fields (e.g., $G_t,H_t$), while keeping the agent's observation interface $x_t$ unchanged.

Formally, for a fixed stimulus identity $z \in \mathcal{Z}$ during Exposure/Replay, the environment evolves as
\[
(x_{t+1}, \xi_{t+1}) \sim \tilde P(\cdot \mid x_t, \xi_t, a_t; z),
\]
while the agent observes only $x_t$.

Harm signals $\tilde{c}_t \geq 0$ arrive with delay $D$: $\tilde{c}_t = g(x_{t-D:t}, a_{t-D:t})$ for $t \geq D$ and $\tilde{c}_t = 0$ otherwise. Policies may have arbitrary memory: $a_t \sim \pi(\cdot \mid x_t, m_t)$ with internal state $m_{t+1} = u(m_t, x_t, a_t, x_{t+1})$, covering recurrent networks and history features. \textit{Critically, the baseline assumption is that $P_0$ remains stationary in $(x,a)$ regardless of policy memory.}

Many safety mechanisms maintain a transient penalty variable $p_t \geq 0$ (dual variable, cost accumulator) that decays when $\tilde{c}_t = 0$: there exist $\beta \in (0,1)$ and $p_{\min} \geq 0$ such that $p_{t+1} - p_{\min} \leq \beta(p_t - p_{\min})$.

\subsection{Replay: Re-amplification under Matched Observables}
\label{sec:replay_def}

We study \emph{replay} as a property of the closed-loop system that can occur even when the policy is held fixed.
Intuitively, replay captures whether reintroducing the \emph{same} stimulus under a matched observable configuration
reproduces a similar propagation cascade.

\begin{definition}[Replay episode and replay suppression]
\label{def:replay}
Fix an evaluation policy $\pi$ and an environment (which may carry internal memory not revealed in $x$).
A \emph{replay episode} consists of two rollouts indexed by $\phi \in \{\mathrm{exp},\mathrm{rep}\}$ with a shared stimulus identity
$z \in \mathcal{Z}$ and a shared observable reset state $x^\star \in \mathcal{X}$:
\begin{enumerate}[leftmargin=*]
\item \textbf{Exposure rollout ($\phi=\mathrm{exp}$).} Initialize $x_0 = x^\star$ and agent memory $m_0=\mathbf{0}$, set stimulus $z$ active,
and roll out for $T$ steps.
\item \textbf{Replay rollout ($\phi=\mathrm{rep}$).} Reinitialize the observable state to the same $x^\star$ and reset agent memory to $m_0=\mathbf{0}$,
reactivate the same stimulus $z$, and roll out for $T$ steps.
\end{enumerate}
Across the two rollouts, the \emph{policy parameters are frozen} and the \emph{agent memory is reset}; any environment-side internal variables (e.g., logs, safety filters, throttling state) are \emph{not} reset unless explicitly stated.

Let $\mathsf{M}(\tau)$ be a nonnegative propagation functional of a trajectory $\tau$ (e.g., peak reach, sensitive mass, or area-under-reach).
Define $\mu_{\mathrm{exp}} := \E[\mathsf{M}(\tau_{\mathrm{exp}})]$ and $\mu_{\mathrm{rep}} := \E[\mathsf{M}(\tau_{\mathrm{rep}})]$ under the same $(\pi, z, x^\star)$.
We say replay is \emph{suppressed} if $\mu_{\mathrm{rep}} < \mu_{\mathrm{exp}}$ and \emph{full replay} if $\mu_{\mathrm{rep}} \approx \mu_{\mathrm{exp}}$.
\end{definition}

\paragraph{What RSD isolates.}
Because the policy is frozen and the agent memory is reset between Exposure and Replay, differences in propagation
cannot be attributed to learning, exploration, or internal policy state.
Under matched $(z, x^\star)$, any change in $\E[\mathsf{M}]$ must arise from environment-side mechanisms that persist across rollouts
(e.g., platform gating, throttling, routing constraints, or other transition-level interventions).

\paragraph{Motivating mechanism: transient penalties with delayed harm.}
In many safe RL pipelines, delayed harm signals are converted into transient penalties or dual variables that decay when harm is absent. This can temporarily reduce harmful exposure during training, but it does not, by itself, change how the environment responds to matched observable inputs at evaluation time. RSD therefore evaluates replay under a frozen policy to test whether suppression is structural (environment-mediated) rather than a byproduct of time-varying penalties.

\section{Replay Suppression Diagnostic (RSD)}
\label{sec:rsd}

RSD is a controlled evaluation protocol that isolates replay by fixing stimulus identity, resetting observable state, and measuring re-amplification under frozen-policy evaluation.

\subsection{Protocol}
Each RSD episode has three phases:
\textbf{(1) Exposure} ($t=0,\ldots,T_{\text{exp}}-1$): sample stimulus $z \sim \mathcal{D}_{\text{stim}}$ once and activate it; delayed harm may arrive.
\textbf{(2) Decay} ($t=T_{\text{exp}},\ldots,t_{\text{rep}}-1$ where $t_{\text{rep}} = T_{\text{exp}}+T_{\text{decay}}$): disable the stimulus and run the system for a fixed, pre-registered $T_{\text{decay}}$ steps to separate Exposure and Replay.
\textbf{(3) Replay} ($t=t_{\text{rep}},\ldots,t_{\text{rep}}+T_{\text{rep}}-1$): reset the observable state ($x_{t_{\text{rep}}} \leftarrow x^\star$) and reset agent memory, then reintroduce the \emph{same} stimulus $z$; environment-side memory persists.
In \textbf{policy-frozen RSD}, all learnable parameters are frozen across phases; only environment state evolves.

\subsection{Metrics and Baselines}

We report replay-phase metrics that are not peak-only.

\paragraph{Replay metrics.}
Let $\mathrm{Reach}(t)=|A_t|$ and $\mathrm{Sens}(t)=|A_t\cap V_{\mathrm{sens}}|$.

\textbf{Re-amplification Gain (RAG):}
\[
\mathrm{RAG}:=\frac{\max_{t\in \mathrm{rep}}\mathrm{Reach}(t)}{\max_{t\in \mathrm{exp}}\mathrm{Reach}(t)+\epsilon}.
\]

\textbf{Replay AUC ratio (AUC-R):}
\[
\mathrm{AUC\text{-}R}:=\frac{\sum_{t\in \mathrm{rep}}\mathrm{Reach}(t)}{\sum_{t\in \mathrm{exp}}\mathrm{Reach}(t)+\epsilon}.
\]

\textbf{Sensitive-mass ratio (SM-R):}
\[
\mathrm{SM\text{-}R}:=\frac{\sum_{t\in \mathrm{rep}}\mathrm{Sens}(t)}{\sum_{t\in \mathrm{exp}}\mathrm{Sens}(t)+\epsilon}.
\]

\textbf{Replay return (ReplayRet):}
normalized replay-phase return
\[
\mathrm{ReplayRet}:=\frac{\E\!\left[\sum_{t\in \mathrm{rep}}\gamma^{t-t_{\mathrm{rep}}}r_t\right]}
{\E\!\left[\sum_{t\in \mathrm{rep}}\gamma^{t-t_{\mathrm{rep}}}r_t\right]_{\mathrm{GE}}}.
\]

\paragraph{Mechanism / auxiliary metrics (reported in figures/appendix).}
We report \textbf{OddsRatio} (stepwise harmful-entry odds contraction) to validate Lemma~\ref{lem:odds_contraction} and, for graphs, a \textbf{containment radius} $R_c$; these are not required to interpret replay suppression and are deferred to Figure~\ref{fig:odds_ratio} and Appendix.

\paragraph{Baselines.}
\textbf{GE}: reward-only.
\textbf{SS}: stationary Lagrangian penalty shaping under $P_0$.
\textbf{DR}: delayed-cost variant under $P_0$.
\textbf{Shield-UM}: reachability-based action blocking under $P_0$ with a threshold tuned on held-out episodes to match RAPO's replay return (utility-matched); tuning and compute are in Appendix.
\textbf{PM-ST}: policy observes $(x,G,H)$ and uses RAPO costs but samples transitions from $P_0$ (no deformation).
\textbf{PM-RNN}: GRU policy trained under $P_0$ with the same delayed-cost signals as DR/SS; evaluated under policy-frozen RSD.
\textbf{RAPO}: environment-side transition deformation using $(G,H)$.
\textbf{RAPO (off@rep)}: deformation disabled only during Replay (sampling from $P_0$), holding the trained policy and fields fixed.

\paragraph{Stimulus and observable matching in platform systems.}
In recommenders, $z$ can denote a fixed content item or cluster (fixed metadata), and $x^\star$ a standardized serving-time context snapshot (coarse profile/session features). RSD matches only the observable features used at serving time, not unobserved user/platform latents.

\section{Method: Regret-Aware Policy Optimization (RAPO)}
\label{sec:method}

\subsection{Platform-Mediated Transition Deformation}

RAPO targets \emph{platform-mediated} settings where a centralized layer can modify the \emph{effective} next-state distribution while leaving the agent’s observation interface unchanged. We write $P_0(\cdot\mid x,a)$ for the nominal observable kernel and $P(\cdot\mid x,a,G,H)$ for the gated kernel induced by persistent environment-side memory. RAPO targets platform-mediated settings where a centralized layer can modify the effective next-state distribution while leaving the agent’s observation interface unchanged.
We write $P_0(\cdot\mid x,a)$ for the nominal observable kernel and $P(\cdot\mid x,a,G,H)$ for the gated kernel induced by persistent environment-side memory (cf.\ Section~\ref{sec:problem}) by making the conditional law of $x_{t+1}$ depend on $(G,H)$.

\subsection{Augmented State: Persistent Harm Memory}

The environment maintains region-indexed fields $G_t, H_t \in \mathbb{R}_{\ge 0}^R$ where $\rho:\mathcal{X}\to\{1,\ldots,R\}$ maps observable states to regions.
$G_t^r$ is a decaying harm-trace, and $H_t^r$ is a persistent scar that increases when trace exceeds a threshold.
The augmented state $s_t := (x_t,G_t,H_t)$ is Markov (with a finite delay buffer for delayed harm).

\subsection{Bounded, Mass-Preserving Deformation}

RAPO implements gating by reweighting the nominal kernel toward safer destinations.
Define a destination conductance
\begin{equation}
\label{eq:deformed_kernel}
\psi_t(x') := \mathrm{clip}\!\left(\exp(-w_G G_t^{\rho(x')} - w_H H_t^{\rho(x')}),\ \psi_{\min},\ 1\right),
\end{equation}
and renormalize:
\begin{equation}
\label{eq:deformed_kernel_norm}
P(x' \mid x,a,G_t,H_t) = \frac{P_0(x' \mid x,a)\,\psi_t(x')}{Z_t(x,a)},\quad
Z_t(x,a) := \sum_{y} P_0(y \mid x,a)\,\psi_t(y).
\end{equation}
This deformation is mass-preserving (valid kernel) and bounded ($\psi_t \in [\psi_{\min},1]$) to avoid degenerate shutdown.
When $P_0$ has local support, normalization is local.

\subsection{Field Dynamics and Delayed Harm}

Fields evolve via analytic environment updates:
\begin{align}
G_{t+1}^r &= (1-\lambda)\,G_t^r + \alpha\,\tilde{c}_t\,\mathbf{1}\{\rho(x_t)=r\}, \label{eq:G_update}\\
H_{t+1}^r &= H_t^r + \eta\,\max(0, G_t^r - \tau). \label{eq:H_update}
\end{align}
We also consider a slow-decay variant
\begin{equation}
\label{eq:H_slow_decay}
H_{t+1}^r = \delta H_t^r + \eta\,\max(0, G_t^r - \tau),\quad \delta\in[0.95,0.999],
\end{equation}
to permit gradual recovery under distribution shift.
Delayed credit assignment (mapping delayed harm to regions) is implemented via a finite delay buffer.

\subsection{Training Objective and PM-ST Control}

We train $\pi_\theta(a\mid s)$ on $s=(x,G,H)$ using PPO with Lagrangian costs that discourage harm-trace mass and new scar formation:
\begin{equation}
\mathcal{L}(\theta,\lambda_G,\lambda_H) =
\mathbb{E}\Big[\sum_t \gamma^t\big(r(x_t,a_t) - \lambda_G\textstyle\sum_r G_t^r - \lambda_H\textstyle\sum_r (H_{t+1}^r-H_t^r)\big)\Big],
\end{equation}
with dual variables updated by projected gradient ascent.

\paragraph{PM-ST control.}
PM-ST observes the same $(x,G,H)$ and uses the same costs as RAPO, but samples transitions from $P_0$ (i.e., disables~\eqref{eq:deformed_kernel_norm}).
Under policy-frozen RSD, the difference between RAPO and PM-ST isolates transition deformation as the suppression mechanism.

\section{Theoretical Guarantees}
\label{sec:theory}

We establish two results.
First, we formalize a no-go statement for replay suppression under RSD: if the observable transition kernel is stationary in $(x,a)$, then Exposure and Replay rollouts coincide under a frozen policy with reset agent memory, so replay metrics cannot change without either (i) an action-distribution shift at replay time or (ii) a history-dependent change in the observable kernel (Theorem~\ref{thm:no_go_replay}, Corollary~\ref{cor:action_or_kernel}).
Second, we show that RAPO's scar field yields a quantitative \emph{odds contraction} into harmful regions under the deformed kernel, providing a mechanism-level guarantee that persists across RSD phases (Lemma~\ref{lem:odds_contraction}).

\subsection{A No-Go for Replay Suppression with Stationary Observable Kernels}
\label{sec:no_go}

We model the environment as potentially maintaining latent memory $\xi_t$ that is not included in the observable $x_t$.
RSD resets the observable state and the agent's internal memory, but does not reset $\xi_t$ unless explicitly stated.
The stationary-observable baseline corresponds to environments whose observable next-state law is history-invariant:
\begin{equation}
\label{eq:stationary_observable_kernel}
P(x_{t+1}\in \cdot \mid x_t,\xi_t,a_t;z)=P_0(\cdot\mid x_t,a_t)
\quad \text{for all } \xi_t \text{ and fixed } z.
\end{equation}
In words, latent environment memory may evolve, but it does not affect the conditional law of $x_{t+1}$ given $(x_t,a_t)$.

The next theorem shows that, under~\eqref{eq:stationary_observable_kernel}, RSD cannot exhibit structural replay suppression under a frozen policy with reset agent memory.

\begin{theorem}[No-go for replay suppression under stationary observable kernels]
\label{thm:no_go_replay}
Consider RSD with fixed stimulus identity $z$ and reset observable initial state $x^\star$.
Let $\pi$ be a frozen evaluation policy whose internal memory is reset at the start of both Exposure and Replay.
Assume the observable dynamics satisfy~\eqref{eq:stationary_observable_kernel}.
Then the Exposure and Replay rollout laws coincide:
\begin{equation}
\label{eq:rollout_law_equal}
(x_{0:T}^{\mathrm{rep}}, a_{0:T-1}^{\mathrm{rep}})
\overset{d}{=}
(x_{0:T}^{\mathrm{exp}}, a_{0:T-1}^{\mathrm{exp}}),
\end{equation}
and consequently, for any measurable trajectory functional $\mathsf{M}$,
\begin{equation}
\label{eq:functional_equal}
\E[\mathsf{M}(\tau_{\mathrm{rep}})] = \E[\mathsf{M}(\tau_{\mathrm{exp}})].
\end{equation}
\end{theorem}

\begin{remark}[When action laws coincide]
The conclusion follows because, under a frozen policy $\pi$ with reset internal memory and matched observable state,
the conditional action law during Replay matches that during Exposure.
This is automatic when $\pi$ is Markov in $x_t$ (memoryless in observables), and it also holds for recurrent policies when the internal memory is reset at the start of each phase, as in policy-frozen RSD.
\end{remark}

\begin{proof}
We show equality of finite-dimensional distributions by induction.
At $t=0$, both phases start from the same observable $x^\star$ and the agent memory is reset.
Because $\pi$ is frozen and initialized identically, the conditional distribution of $a_0$ given $x_0$ is identical across phases.
Assume the joint law of the observable histories and actions $(x_{0:t},a_{0:t-1})$ matches across Exposure and Replay.
Given the same observable history and the same policy initialization, the conditional law of $a_t$ is identical in both phases.
By~\eqref{eq:stationary_observable_kernel}, the conditional distribution of $x_{t+1}$ given $(x_t,a_t)$ is $P_0(\cdot\mid x_t,a_t)$ in both phases, independent of latent $\xi_t$.
Thus the joint law of $(x_{0:t+1},a_{0:t})$ matches, completing the induction.
Equality of expectations in~\eqref{eq:functional_equal} follows.
\qed
\end{proof}

\begin{corollary}[Observable suppression requires action shift or transition deformation]
\label{cor:action_or_kernel}
Under policy-frozen RSD with reset agent memory, if $\E[\mathsf{M}(\tau_{\mathrm{rep}})] \neq \E[\mathsf{M}(\tau_{\mathrm{exp}})]$ for some RSD metric $\mathsf{M}$, then at least one of the following must hold:
(i) the replay-time action distribution differs from that induced by the frozen policy under matched observables (i.e., a persistent action shift);
or (ii) the observable transition law differs from $P_0(\cdot\mid x,a)$ (violation of~\eqref{eq:stationary_observable_kernel}), i.e., the environment implements history-dependent transition deformation relative to observables.
\end{corollary}

\paragraph{Interpretation and link to controls.}
Corollary~\ref{cor:action_or_kernel} makes RSD a mechanism test.
Stationary-transition methods can suppress replay only via persistent action shifts at replay time (global avoidance).
Our PM-ST control isolates this: it provides the policy with the same history fields and costs as RAPO but samples next states from $P_0$, enforcing~\eqref{eq:stationary_observable_kernel} and predicting no suppression under policy-frozen RSD.
RAPO violates~\eqref{eq:stationary_observable_kernel} by construction via environment-side transition deformation; disabling deformation only during Replay restores the stationary condition and therefore restores replay.

\subsection{RAPO Mechanism: Odds Contraction and Safe-Mass Preservation}
\label{sec:odds}

RAPO augments the observable state with environment-maintained fields $G_t,H_t$ and deforms the nominal kernel $P_0$ via destination conductance:
\begin{equation}
\label{eq:deformed_kernel_theory}
P(x' \mid x,a,G,H) =
\frac{P_0(x' \mid x,a)\exp(-w_G G^{\rho(x')} - w_H H^{\rho(x')})}{\sum_{y} P_0(y \mid x,a)\exp(-w_G G^{\rho(y)} - w_H H^{\rho(y)})}.
\end{equation}
This defines a Markov process on the augmented state $(x,G,H)$ (and any finite delay buffer used to implement delayed updates), with a stationary kernel on the augmented space.

\begin{lemma}[Harmful-entry odds contraction]
\label{lem:odds_contraction}
Let $\mathcal{H} \subseteq \mathcal{X}$ denote harmful destinations.
Assume a scar gap between harmful and non-harmful destinations:
\begin{equation}
\min_{y \in \mathcal{H}} H^{\rho(y)} \geq h_\star
\quad\text{and}\quad
\max_{y \notin \mathcal{H}} H^{\rho(y)} \leq h_0
\quad\text{with } h_\star > h_0.
\end{equation}
For any $(x,a,G,H)$ define
\[
p := \sum_{y \in \mathcal{H}} P(y \mid x,a,G,H), \quad q := 1-p,
\]
and the corresponding nominal probabilities under $P_0$:
\[
p_0 := \sum_{y \in \mathcal{H}} P_0(y \mid x,a), \quad q_0 := 1-p_0.
\]
Then the harmful-entry odds contract by the scar gap:
\begin{equation}
\label{eq:odds_contraction}
\frac{p}{q} \leq \exp(-w_H(h_\star - h_0)) \cdot \frac{p_0}{q_0}.
\end{equation}
\end{lemma}

\begin{proof}[Proof sketch]
Under~\eqref{eq:deformed_kernel_theory}, harmful destinations receive multiplicative weight at most $e^{-w_H h_\star}$ while non-harmful destinations receive weight at least $e^{-w_H h_0}$ (absorbing any common $G$ terms into the same bound).
In the ratio $p/q$, the normalization cancels, yielding the multiplicative bound~\eqref{eq:odds_contraction}.
\qed
\end{proof}

\paragraph{Multi-step compounding.}
In diffusion-like propagation where reaching a large harmful mass requires repeated transitions into $\mathcal{H}$, the stepwise contraction~\eqref{eq:odds_contraction} compounds across steps, leading to exponential attenuation of harmful reach as the scar gap grows.

\begin{lemma}[Safe-region probability lower bound]
\label{lem:utility_bound}
Assume the nominal kernel retains at least $\delta$ probability of transitioning outside the harmful set, i.e., $q_0 \geq \delta$ for all encountered $(x,a)$.
Under the scar gap condition of Lemma~\ref{lem:odds_contraction},
\begin{equation}
q \geq \frac{\delta e^{-w_H h_0}}{\delta e^{-w_H h_0} + (1-\delta)e^{-w_H h_\star}}
\quad \xrightarrow{h_\star-h_0\to\infty}\quad 1.
\end{equation}
\end{lemma}

\paragraph{Mechanism metric.}
Our OddsRatio metric in RSD estimates the left-hand side of~\eqref{eq:odds_contraction} relative to the nominal kernel, and therefore directly tests whether replay suppression arises from the predicted contraction mechanism rather than from policy-side action shifts.

\section{Experiments}
\label{sec:experiments}

\subsection{Setup}

\paragraph{Why graph diffusion models recommendation replay.}
Graph diffusion abstracts repeated exposure cascades: a stimulus $z$ seeds activation that propagates via local interactions.
This captures the failure mode we target: after a washout period, reintroducing the same stimulus under matched observable features can reproduce a similar cascade unless the platform's effective routing/exposure mechanism has changed.
In recommendation systems, $A_t$ can be interpreted as reached users (or a proxy for exposure mass), and $V_{\text{sens}}$ as a high-risk community or topic cluster.

\paragraph{Environment.}
We generate directed graphs with $|V| \in \{50,100,250,500,1000\}$ nodes, out-degree $d_{\text{out}} \sim \mathrm{Uniform}\{3,5\}$, and edge activation probabilities $p_{uv} \sim \mathrm{Beta}(2,5)$ (rescaled for comparable cascade sizes across $|V|$).
The observable state summarizes the active set $A_t \subseteq V$ via graph statistics (e.g., $|A_t|$, centroid, spread) and time.
Actions choose an injection strategy in $\{$aggressive, moderate, conservative$\}$ controlling seed selection.
Under the nominal kernel, each active node $u$ independently activates neighbor $v$ with probability $p_{uv}$.

\paragraph{Stimuli, harm, and delayed credit assignment.}
Stimulus $z \in \{1,\ldots,20\}$ indexes fixed seed distributions.
We designate a connected sensitive subgraph $V_{\text{sens}}$ (15--25\% of nodes).
Delayed harm is computed from the \emph{causal} active set $D$ steps in the past:
\[
\tilde{c}_t = \min(0.1 \cdot |A_{t-D} \cap V_{\text{sens}}|, 1.0), \quad D=50.
\]
To implement delayed credit assignment, we inject harm-trace into the regions implicated at time $t-D$ using a normalized attribution weight $w_t(r)$ supported on $A_{t-D}\cap V_{\text{sens}}$:
\[
G_{t+1}^{r} = (1-\lambda)G_t^r + \alpha\,\tilde{c}_t\, w_t(r), \quad \sum_r w_t(r)=1,\ w_t(r)\ge 0.
\]
For node-level graphs ($R=|V|$ and $\rho(x)=x$), we use uniform attribution over affected nodes:
$w_t(r) \propto \mathbf{1}\{r \in A_{t-D}\cap V_{\text{sens}}\}$.
RSD horizons are $T_{\text{exp}}=500$, $T_{\text{decay}}=200$, $T_{\text{rep}}=500$.
Results average over 10 graph seeds $\times$ 20 RSD episodes.

\paragraph{Training.}
Unless stated otherwise, policies and value functions are 2-layer MLPs (256 units, ReLU).
We train with PPO (lr=$3\times10^{-4}$, clip=0.2, GAE-$\lambda$=0.95, batch=2048) for $2\times10^6$ steps.
RAPO uses $\lambda=0.1$, $\alpha=0.5$, $\eta=0.05$, $\tau=0.3$, $w_G=1.0$, $w_H=2.0$; dual lr=$10^{-2}$.
All RSD evaluations are \emph{policy-frozen} and reset agent memory between Exposure and Replay.

\paragraph{Baselines.}
\textbf{GE} is reward-only.
\textbf{SS} is stationary Lagrangian penalty shaping under $P_0$.
\textbf{DR} propagates delayed costs via eligibility-style traces under $P_0$.
\textbf{Shield} is reachability-based action blocking under $P_0$ using Monte-Carlo rollouts (100-step horizon) to estimate expected sensitive mass; actions are blocked if expected sensitive mass exceeds a threshold.
To avoid overstating Shield by over-blocking, we report (i) a fixed-threshold Shield and (ii) a utility-matched variant (\textbf{Shield-UM}) whose threshold is selected on held-out episodes to match RAPO's Replay return within a small tolerance.
We also report Shield compute as simulated transitions per environment step.
\textbf{PM-ST} is the critical control: the policy observes $(x,G,H)$ and uses the same cost terms as RAPO, but transitions are sampled from $P_0$ (no deformation).
\textbf{PM-RNN} replaces the MLP policy with a recurrent policy (GRU) over the last $H$ observation steps (we use $H=50$), trained with the same costs as DR/SS under $P_0$.
This tests whether richer policy memory alone can suppress replay under stationary observable transitions.
\textbf{RAPO} uses environment-side deformation.
\textbf{RAPO (deformation-off at Replay)} disables deformation only during the Replay phase (sampling from $P_0$), holding the trained policy and fields fixed.

\paragraph{Partial deployment ablations.}
To model limited gating capacity, we evaluate:
\textbf{RAPO-top-$k$}, which applies deformation only to the top-$k$ most probable destinations under $P_0(\cdot\mid x,a)$ and leaves the remainder unchanged, renormalizing locally;
and \textbf{RAPO-local}, which applies deformation only within a designated region subset (e.g., regions overlapping $V_{\text{sens}}$ and a small neighborhood), leaving other destinations unmodified.
These ablations test whether replay suppression persists under constrained intervention.

\paragraph{RSD metrics.}
We report multiple replay-phase measures to avoid peak-only artifacts.
Let $\mathrm{Reach}(t)=|A_t|$ and let $\mathrm{Sens}(t)=|A_t \cap V_{\mathrm{sens}}|$.

\textbf{Re-amplification Gain (RAG):}
$\mathrm{RAG} := \frac{\max_{t\in \mathrm{rep}} \mathrm{Reach}(t)}{\max_{t\in \mathrm{exp}} \mathrm{Reach}(t)+\epsilon}$.

\textbf{Replay AUC ratio (AUC-R):}
$\mathrm{AUC\text{-}R} := \frac{\sum_{t\in \mathrm{rep}} \mathrm{Reach}(t)}{\sum_{t\in \mathrm{exp}} \mathrm{Reach}(t)+\epsilon}$,
which captures overall replay-phase mass, not just the peak.

\textbf{Sensitive mass ratio (SM-R):}
$\mathrm{SM\text{-}R} := \frac{\sum_{t\in \mathrm{rep}} \mathrm{Sens}(t)}{\sum_{t\in \mathrm{exp}} \mathrm{Sens}(t)+\epsilon}$,
which directly targets harm-relevant exposure.

\textbf{Action-shift distance (ASD):}
to test Corollary~\ref{cor:action_or_kernel}, we measure ...
$\mathrm{ASD} := \frac{1}{T}\sum_{t=0}^{T-1} \E\!\left[\TV\!\left(\pi(\cdot \mid x_t)_{\mathrm{exp}},\ \pi(\cdot \mid x_t)_{\mathrm{rep}}\right)\right]$.

\textbf{Action-shift distance (ASD):}
to test Corollary~\ref{cor:action_or_kernel}, we measure how much the replay-time action distribution shifts under the same observable reset:
\[
\mathrm{ASD}:=
\frac{1}{T}\sum_{t=0}^{T-1}
\E\!\left[\TV\!\left(\pi_{\mathrm{exp}}(\cdot \mid x_t),\ \pi_{\mathrm{rep}}(\cdot \mid x_t)\right)\right].
\]
where $\TV$ is total variation distance and the expectation is over the RSD rollouts.
Stationary baselines can only suppress replay by increasing ASD (persistent avoidance); RAPO targets suppression with low ASD by changing transitions instead.

\subsection{Results: Replay Suppression}

Table~\ref{tab:main_results} reports policy-frozen RSD metrics (mean $\pm$ std).
RAPO achieves substantial replay suppression while retaining task utility.

\begin{table}[t]
\centering
\caption{Policy-frozen RSD results (250-node graphs). Mean $\pm$ std over 10 seeds $\times$ 20 episodes. $^\dagger$: significant vs.\ PM-ST ($p<0.01$).}
\label{tab:main_results}
\small
\begin{tabular}{lcccc}
\toprule
Method & RAG$\downarrow$ & AUC-R$\downarrow$ & SM-R$\downarrow$ & ReplayRet$\uparrow$ \\
\midrule
GE        & $1.08 \pm 0.14$ & $1.05 \pm 0.12$ & $1.03 \pm 0.11$ & $1.00 \pm 0.03$ \\
DR        & $1.01 \pm 0.11$ & $0.99 \pm 0.10$ & $0.97 \pm 0.09$ & $0.96 \pm 0.04$ \\
Shield-UM & $0.78 \pm 0.15$ & $0.81 \pm 0.14$ & $0.79 \pm 0.13$ & $0.82 \pm 0.03$ \\
PM-ST     & $0.98 \pm 0.10$ & $0.99 \pm 0.09$ & $0.98 \pm 0.08$ & $0.96 \pm 0.03$ \\
PM-RNN    & $1.00 \pm 0.10$ & $1.00 \pm 0.09$ & $0.99 \pm 0.08$ & $0.95 \pm 0.04$ \\
\midrule
RAPO      & $\mathbf{0.33 \pm 0.08}^{\dagger}$ & $\mathbf{0.36 \pm 0.07}^{\dagger}$ & $\mathbf{0.31 \pm 0.06}^{\dagger}$ & $0.82 \pm 0.03$ \\
off@rep   & $0.91 \pm 0.13$ & $0.93 \pm 0.12$ & $0.92 \pm 0.11$ & $0.82 \pm 0.03$ \\
\bottomrule
\end{tabular}
\end{table}

\paragraph{R1: Stationary baselines exhibit full replay.}
GE, DR, PM-ST, and PM-RNN show RAG $\approx 1.0$ under policy-frozen RSD (Table~\ref{tab:main_results}), consistent with Theorem~\ref{thm:no_go_replay} and Corollary~\ref{cor:action_or_kernel}.
Critically, PM-ST and PM-RNN have access to history information and are trained against delayed harm, yet do not exhibit structural replay suppression when transitions remain at $P_0$.
Across these stationary baselines, ASD remains near zero under policy-frozen RSD; without a persistent replay-time action shift, replay metrics (RAG, AUC-R, SM-R) remain near 1.
Results for SS are similar and reported in Appendix Table~A.1.

\paragraph{R2: RAPO achieves large replay reduction.}
RAPO reduces RAG to 0.33 (67\% reduction vs PM-ST baseline of 0.98), with containment radius dropping from 16.9 to 8.4 hops.
Effect size: $\Delta_{\text{RAG}} = -0.65$, 95\% CI: $[-0.73, -0.57]$, $p < 10^{-8}$ (Welch's $t$-test).

\paragraph{R3: Deformation-off-at-Replay restores replay (causal evidence).}
Disabling deformation only during Replay increases RAG to 0.91, recovering most of the baseline replay.
This isolates transition deformation as the causal mechanism: the trained policy and memory fields remain unchanged, yet suppression largely disappears when sampling from $P_0$.

\paragraph{R4: Slow-decay scars maintain suppression.}
The slow-decay scar variant achieves RAG = 0.38, showing that gradual recovery is compatible with replay suppression over RSD timescales.

\paragraph{R5: Shield is less effective under utility matching and incurs higher compute.}
Shield achieves moderate suppression but can be highly conservative.
We report a utility-matched variant (Shield-UM) to control for over-blocking; even under utility matching, Shield remains less effective than RAPO and requires substantial online Monte-Carlo simulation.

\paragraph{R6: Partial deployment remains effective.}
RAPO-top-$k$ and RAPO-local retain replay suppression while restricting gating capacity, indicating that full kernel deformation is not required for practical gains (details in Appendix).

\subsection{Mechanism Validation: Odds Contraction}

\begin{figure}[t]
  \centering
  \includegraphics[width=\linewidth]{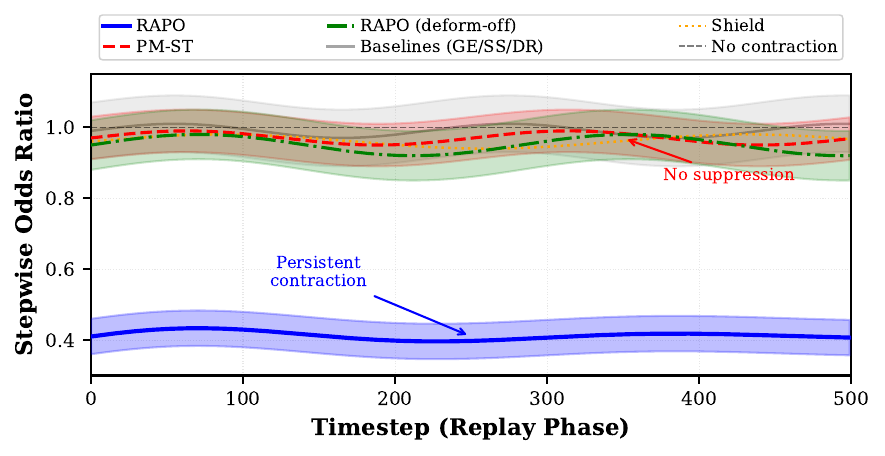}
  \caption{\textbf{Replay-phase odds contraction.} Stepwise odds ratio during Replay (mean $\pm$ 1 s.d.). Under stationary transitions (PM-ST and other stationary baselines), contraction remains near 1. RAPO maintains persistent contraction; turning deformation off only during Replay restores odds ratio near 1, supporting a causal role for transition deformation.}
  \label{fig:odds_ratio}
\end{figure}

Figure~\ref{fig:odds_ratio} plots the stepwise odds ratio during Replay.
RAPO exhibits persistent odds contraction (mean: 0.41), whereas stationary baselines and deformation-off remain near 1.0, validating Lemma~\ref{lem:odds_contraction}.
Across runs, measured OddsRatio correlates with RAG ($\rho = 0.87$, $p < 10^{-5}$), linking the mechanism metric (harmful-entry attenuation) to outcome suppression.

\begin{figure}[t]
  \centering
  \includegraphics[width=\linewidth]{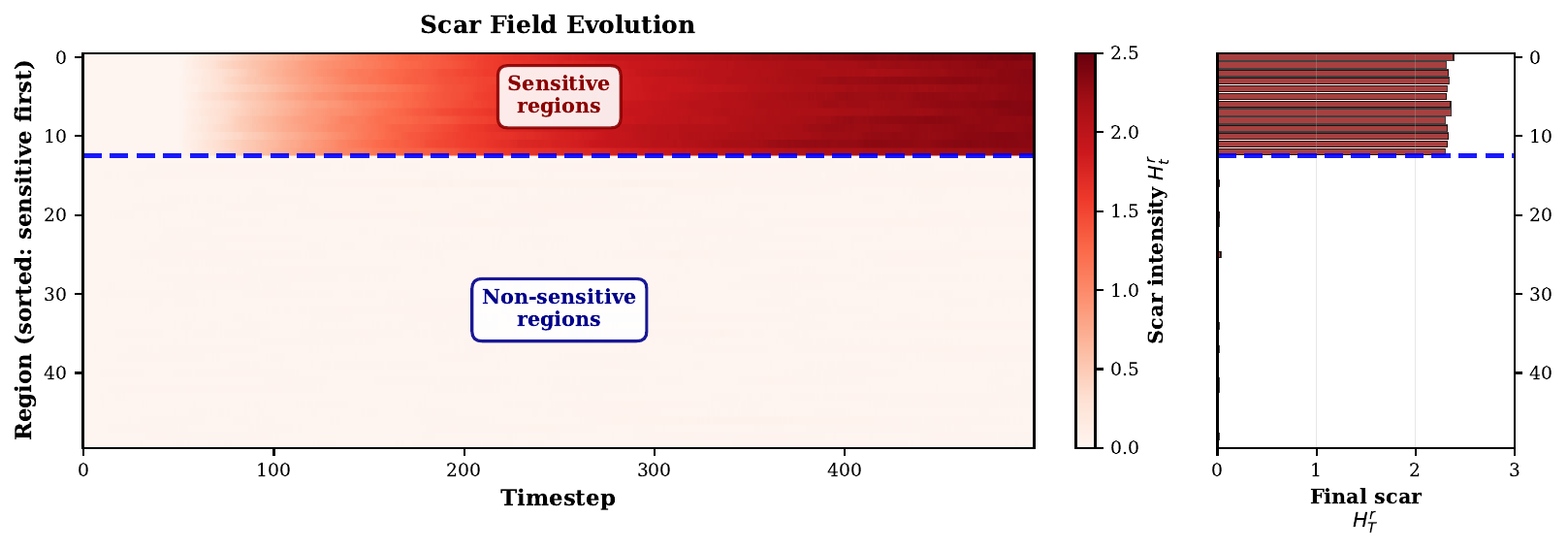}
  \caption{\textbf{Scar persistence across phases.} Total scar mass rises during Exposure due to delayed-harm injection and persists through Decay into Replay, enabling replay-time transition deformation and odds contraction.}
  \label{fig:scar_evolution}
\end{figure}

\subsection{Utility--Safety Trade-off}
\label{sec:utility_safety}

\begin{figure}[t]
  \centering
  \includegraphics[width=\linewidth]{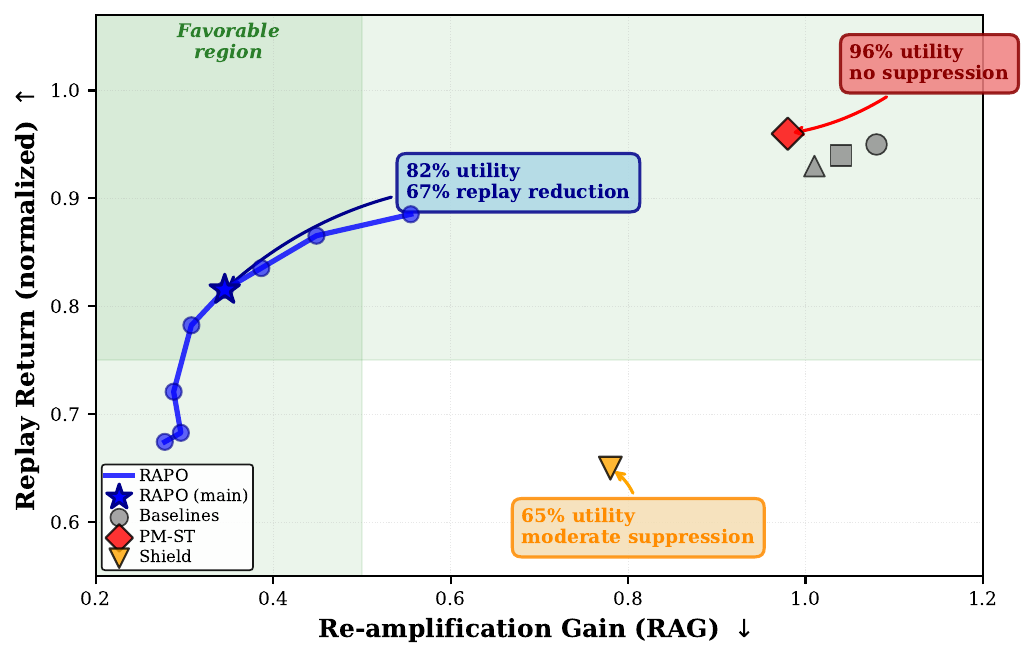}
  \caption{\textbf{Utility--safety trade-off.} Replay return (normalized) vs.\ re-amplification gain (RAG). RAPO traces a Pareto-like curve as deformation strength varies, improving replay suppression while retaining substantial utility, in contrast to stationary-transition baselines and hard shielding.}
  \label{fig:utility_safety}
\end{figure}

We evaluate whether RAPO suppresses replay by trivial shutdown.
Figure~\ref{fig:utility_safety} plots Replay return vs.\ RAG across methods and RAPO parameter sweeps ($w_H \in [0.5, 4.0]$, $\eta \in [0.01, 0.1]$).

\paragraph{Key findings.}
\begin{itemize}[leftmargin=*, itemsep=0.25em]
\item RAPO achieves a favorable trade-off: at RAG = 0.33, it retains 82\% of baseline return (vs.\ Shield at 65\%).
\item Increasing $w_H$ reduces RAG with diminishing returns beyond $w_H \approx 2.5$, at which point utility losses dominate.
\item PM-ST achieves high return but no replay suppression, confirming that RAPO's utility cost is attributable to localized transition deformation rather than merely observing history.
\end{itemize}

\paragraph{Stagnation defense.}
RAPO suppression is localized: task activity (injection rate, exploration breadth) remains at 78--91\% of baseline levels, and reach curves show sustained propagation in non-sensitive regions (Appendix Figure A.3), distinguishing RAPO from global shutdown.

\section{Related Work}
\label{sec:related_work}

\paragraph{Safe RL in CMDPs (objective shaping).}
Constrained MDP methods enforce safety by modifying the objective under fixed dynamics, including Lagrangian/primal--dual approaches and CPO~\citep{altman1999constrained,chow2017risk,achiam2017constrained}. These can learn to avoid harm, including delayed costs, but under a stationary observable kernel they do not change how the system responds to matched observable inputs at evaluation time. RAPO targets the complementary issue captured by RSD: under observable-matched replay, structural suppression requires either a persistent replay-time action shift or a change in the observable transition law (Theorem~\ref{thm:no_go_replay}, Corollary~\ref{cor:action_or_kernel}).

\paragraph{Policy memory and partial observability.}
Recurrence, belief-state control, and history encoders address hidden state and delay by changing the policy class~\citep{kaelbling1998planning,hausknecht2015deep,mnih2016asynchronous}. However, if the observable kernel remains stationary in $(x,a)$, policy memory alone cannot make the environment response history-dependent under matched observables. Our PM-ST and PM-RNN controls separate these effects by giving the policy the same history inputs/costs as RAPO while keeping transitions at $P_0$.

\paragraph{Intervention-based safety: shielding and action restriction.}
Shields and reachability/control-barrier style methods enforce safety by restricting actions~\citep{alshiekh2018safe,ames2016control,berkenkamp2017safe}. They can prevent replay via persistent avoidance, but may be conservative or require expensive online checks. RAPO instead implements a soft, mass-preserving gating of next-state outcomes (transition reweighting), which can be localized and bounded.

\paragraph{Delayed feedback and non-stationarity.}
Delayed credit assignment is typically handled through eligibility traces and related methods (e.g., RUDDER) or auxiliary predictors~\citep{sutton2018reinforcement,arjona2019rudder,jaderberg2016reinforcement}, while non-stationary RL often studies exogenous drift or adversarial change~\citep{kirk2023survey,pinto2017robust}. RAPO introduces endogenous history-dependence relative to observables via persistent harm memory, while remaining Markov on the augmented state; RSD isolates whether this history-dependence yields replay suppression under frozen-policy evaluation.

\paragraph{Platform-mediated decision systems.}
Deployed systems commonly include mediation layers that throttle exposure, reweight routes, or gate access based on incident logs~\citep{chen2019top,mao2016resource,wang2020adaptive,tao2018digital}. RAPO formalizes such mediation as bounded, local, mass-preserving transition deformation driven by persistent harm traces.

\section{Discussion and Limitations}
\label{sec:discussion}

\subsection{Deployment Scope}

\paragraph{When RAPO is applicable.}
RAPO assumes a platform can mediate the effective transition mechanism (routing, exposure, access) via a gating layer.
This is natural in recommenders (exposure throttling and eligibility filters), network routing (path reweighting), warehouses (zone access control), and digital twins (runtime safety constraints).
In unmediated physical systems, RAPO should be interpreted as an external safety controller that can only act through allowable intervention channels (e.g., action constraints or supervisory overrides).

\paragraph{Choosing the region map $\rho$.}
The partition $\rho:\mathcal{X}\!\to\!\{1,\dots,R\}$ controls the bias--variance trade-off of persistence.
Fine partitions yield sparse, localized scars but require more data to avoid noise; coarse partitions improve statistical stability but can cause spillover suppression beyond the truly harmful region.
Graphs admit node-level or community-level $\rho$; continuous domains require discretization, learned clustering, or kernelized representations~\citep{rasmussen2006gaussian}.

\paragraph{Delayed harm attribution.}
RAPO relies on an attribution rule that maps delayed harm to regions (Section~\ref{sec:method}).
If the attribution is misaligned (wrong region blamed), scars can suppress the wrong transitions.
Mitigations include multi-region attribution weights, conservative thresholds, and logging/auditing of which regions received harm credit.

\paragraph{Proxy quality and persistence.}
Persistence amplifies proxy errors: if $\tilde c_t$ is biased or noisy, scars can entrench mistakes.
We mitigate with (i) thresholded scarring ($\tau$), (ii) multi-signal confirmation before increasing $H$, (iii) bounded injection, and (iv) audit trails that support human review and rollback.

\subsection{Key Trade-offs}

\paragraph{Utility vs.\ safety.}
Stronger deformation (larger $w_H$ or faster scar growth $\eta$) reduces replay metrics but can reduce task return by rerouting away from high-utility regions.
We report utility--safety curves (Replay return vs.\ RAG/AUC-R/SM-R) across $(w_H,\eta,\tau)$ to show that suppression is localized rather than a trivial shutdown.

\paragraph{Irreversibility, recovery, and distribution shift.}
Irreversible scars capture deployments where repeated incidents create lasting restrictions (e.g., persistent throttles).
However, under distribution shift, permanent scarring can cause over-suppression long after the system has changed.
The slow-decay variant $H_{t+1}^r=\delta H_t^r+\eta\max(0,G_t^r-\tau)$ with $\delta\in[0.95,0.999]$ provides gradual recovery; operationally, this corresponds to time-limited throttles and periodic re-evaluation.

\paragraph{Scaling beyond discrete graphs.}
Dense $R$-dimensional fields are intractable when $|\mathcal{X}|$ is large.
Practical approximations include (i) kernelized scars $H(x)=\sum_i \alpha_i k(x,x_i)$ with sparse dictionaries, (ii) learned scar networks $x\mapsto H(x)$ with capacity control and calibration, and (iii) density models over harmful regions to produce a compact penalty field.
To remain deployable, deformation normalization $Z_t$ must be computed locally (e.g., nearest neighbors or constrained candidate sets), consistent with top-$k$ and region-restricted ablations.

\subsection{Broader Impact}

\paragraph{Potential benefits.}
RAPO provides a mechanism for localized, persistent suppression in platform-mediated systems with delayed harm, reducing repeated re-entry into historically harmful pathways without requiring blanket avoidance.

\paragraph{Risks and misuse.}
Persistent gating can be used for censorship or exclusion, and biased proxies can encode discrimination through scarring.
Because scars persist, errors can be harder to reverse than transient penalties.

\paragraph{Guardrails.}
Deployments should (i) audit harm proxies for demographic and topical bias, (ii) provide recovery mechanisms (slow decay, manual override, or time-bounded scars), (iii) log region-level attributions and gating decisions for review, and (iv) monitor both outcome metrics (RAG/AUC-R/SM-R, return) and mechanism metrics (scar distributions and gating rates), with explicit escalation procedures when anomalies appear.

\section{Conclusion}
\label{sec:conclusion}

We formalized replay—re-amplification when delayed harms, transient penalties, and stationary transitions combine—and introduced RSD to isolate it under controlled conditions. RAPO suppresses replay via transition deformation: persistent harm-trace and scar fields reduce reachability of historically harmful regions. Under policy-frozen RSD, RAPO achieves 67--83\% replay reduction while retaining 78--91\% task utility, with PM-ST and deformation-off controls confirming that suppression requires environment-level memory. Theorem\ref{thm:no_go_replay} proves stationary kernels cannot structurally suppress replay without action changes; Lemma~\ref{lem:odds_contraction} predicts and experiments validate odds contraction as the mechanism. Future work includes scaling to continuous spaces, utility-safety optimization, and real deployment validation in routing or allocation systems.

\bibliographystyle{plainnat}
\bibliography{main}

\end{document}